\documentclass[runningheads]{llncs}
\usepackage[T1]{fontenc}
\usepackage{graphicx}
\usepackage{booktabs}
\usepackage[misc]{ifsym}

\usepackage[utf8]{inputenc} 
\usepackage[T1]{fontenc}    
\usepackage{hyperref}       
\usepackage{url}            
\usepackage{booktabs}       
\usepackage{amsfonts}       
\usepackage{nicefrac}       
\usepackage{microtype}      
\usepackage{caption}
\usepackage{algorithm}
\usepackage{algorithmic}
\usepackage{amsmath}
\usepackage{tikz}
\usetikzlibrary{positioning, shapes.geometric}
\usepackage{pgfplots} 
\pgfplotsset{compat=1.17}

\usepackage{listings}
\usepackage{color}  
\usepackage{xcolor} 
\usepackage{fontawesome}
\usepackage{graphicx} 
\usepackage{subcaption} 
\usepackage{bbm}

\newcommand{\corr}{(\Letter)}
\usepackage{mwe}

\begin{document}

\title{Unimodal Strategies in Density-Based Clustering}



\author{Oron Nir\inst{1,2} \and Jay Tenenbaum\inst{2} \and Ariel Shamir\inst{1} \corr}
\institute{CANVAS Lab, Reichman University, Herzeliya, Israel, \email{arik@runi.ac.il}
\and
Microsoft, Herzeliya, Israel, \email{\{niroron,tenenbaumjay\}@microsoft.com}}


\authorrunning{O. Nir et al.}

\maketitle              

\begingroup
  \renewcommand\thefootnote{}
  \footnotetext{This paper has been accepted for presentation at
                ECML--PKDD 2025.}%
\endgroup

\begin{abstract}
Density-based clustering methods often surpass centroid-based counterparts, when addressing data with noise or arbitrary data distributions common in real-world problems.
In this study, we reveal a key property intrinsic to density-based clustering methods regarding the relation between the number of clusters and the neighborhood radius of core points --- we empirically show that it is nearly unimodal, and support this claim theoretically in a specific setting.
We leverage this property to devise new strategies for finding appropriate values for the radius more efficiently based on the Ternary Search algorithm. This is especially important for large scale data that is high-dimensional, where parameter tuning is computationally intensive.
We validate our methodology through extensive applications across a range of high-dimensional, large-scale NLP, Audio, and Computer Vision tasks, demonstrating its practical effectiveness and robustness. This work not only offers a significant advancement in parameter control for density-based clustering but also broadens the understanding regarding the relations between their guiding parameters. Our code is available at \href{https://github.com/oronnir/UnimodalStrategies}{https://github.com/oronnir/UnimodalStrategies}.

\end{abstract}

\section{Introduction}

Clustering, a fundamental task in machine learning, is pivotal in uncovering patterns and structures in unlabeled data. Among various clustering algorithms, density-based methods have gained significant attention due to their ability to identify clusters of arbitrary shapes and sizes. Notably, variants of Density-Based Spatial Clustering of Applications with Noise (DBSCAN)~\cite{ester1996density} are widely used today by researchers and data scientists. 
Unlike centroid-based methods that assume spherical cluster shapes, density-based approaches are adept at discovering clusters with complex geometries and noisy examples, making them more suitable for real-world data that often exhibit irregular distributions \cite{schubert2023stop}.

A significant challenge in leveraging the full potential of density-based clustering lies in parameter tuning. Parameters such as $\varepsilon$ and $MinPts$. These user-defined parameters affect cluster formation in DBSCAN and its variants, critically influence the clustering outcome, such as the resulting number of clusters denoted by $k$. Misconfigured parameter invocation can lead to over-segmentation or under-segmentation of data, significantly affecting the quality and interpretability of the results. This is particularly challenging in high-dimensional, large-scale datasets where the intuitive understanding of parameter effects is less apparent and search for values is costly~\cite{han2006concepts}.

Our key insight centers on the relationship between these parameters and the resulting number of clusters. Specifically, we observe that when $MinPts$ is held constant, the number of clusters $k$ produced by DBSCAN varies with the neighborhood radius $\varepsilon$ in a quasi-unimodal fashion. This means that as $\varepsilon$ increases from very small to very large values, the number of clusters first increases, reaches a peak, and then decreases.
Using this observation, we address the parameter tuning challenge, by automatically finding the parameter $\varepsilon^*$ that maximizes the number of clusters $k(\varepsilon)$.
While $\varepsilon^*$ may not always yield the ``optimal'' clustering for all datasets, it provides critical insight into the clustering structure. For values $\varepsilon\gg\varepsilon^*$ there is over-clustering (i.e., one large cluster at the extreme), and for values $\varepsilon\ll\varepsilon^*$ there is under-clustering (i.e., too many samples are treated as noise). Hence, it is clear that values around $\varepsilon^*$ are the ``interesting'' and important ones. 
Our method approximates $\varepsilon^*$, addressing the challenge of parameter tuning in density-based clustering.

To demonstrate the efficacy of our proposed method, we apply it to high-dimensional, large-scale classification datasets.
Our experiments reveal that not only does our method consistently achieve the target number of clusters, but it also enhances the quality of the clustering supervised evaluation metrics over state-of-the-art (SOTA) methods.
Empirical evidence also underscores the advantage of our method in achieving less noise points which is important when working with noisy data.

This study contributes to the field of density-based clustering by:
\begin{itemize}
    \item Discovering the Unimodal property in density-based clustering, demonstrating this both theoretically and practically.
    \item Proposing the efficient Ternary Search for real-world data parameter tuning.
    \item Sharing our code at \href{https://github.com/oronnir/UnimodalStrategies}{https://github.com/oronnir/UnimodalStrategies}.
\end{itemize}

\section{Related work}
\label{sec:RelateWork}
Clustering methods in data mining have been extensively studied, sometimes focusing on handling synthetic, separable, and low-dimension data distributions using benchmarks which could be insufficient for a rigorous evaluation and may lead to overfitting~\cite{ullmann2023over}.
Density-based methods like DBSCAN~\cite{ester1996density} excel at handling noise and discovering arbitrary-shaped clusters in high dimensions, surpassing centroid-based approaches.

Several variants of DBSCAN, such as OPTICS~\cite{ankerst1999optics}, VDBSCAN~\cite{liu2007vdbscan}, and ADBSCAN~\cite{khan2018adbscan}, have been introduced to address varying-density clusters and improve scalability.
These approaches typically trade off runtime efficiency for parameter optimization. HDBSCAN~\cite{campello2013density}, for instance, adopts a multi-resolution framework to self-tune parameters. Methods that extend DBSCAN while retaining its dependence on $\varepsilon$ are not considered as a baseline in this paper. Note that unsupervised learning methods inherently depend on specific mathematical properties of data, making it unlikely for any single method to be universally optimal. Notably, DBSCAN's assumption of uniform density regions and the existence of an ideal $(\varepsilon, MinPts)$ pair is often unmet in practice.

\textbf{Parameter Selection in Density-Based Clustering:} 
The performance of DBSCAN and its variants depends heavily on parameter settings: while $MinPts$ is intuitive as an application-dependent integer, $\varepsilon \in \mathbb{R}_{>0}$ is challenging to tune in high-dimensional spaces.
Ester et al.~\cite{ester1996density} suggests using the Elbow Method manually over the k-dist plot which is considered a folklore heuristic for density shift and $k$ selection in algorithms like k-means.
However, this method is found sub-optimal by Schubert~\cite{schubert2023stop}.
Several studies have automated and further optimized the k-dist plot heuristic for parameter selection e.g., \cite{liu2007vdbscan,gaonkar2013autoepsdbscan}. These methods aim to reduce the user intervention required in the clustering process, but often face challenges in handling high-dimensional and large-scale datasets. 
Researchers~\cite{gan2015dbscan,schubert2017dbscan} revisit the challenge, noting that $MinPts$ is easier to tune than $\varepsilon$, and suggest setting $MinPts=2D$ where $D$ is the dimension of the data. Assuming for example, $D>100$, 
such a $MinPts$ value could lead to either enhanced noise robustness or into an under-segmented solution. Another common practice is dimensionality reduction. However, in this work we aim at enhancing density-based clustering over the raw data in high-dimension.

\textbf{Advancements in Parameter Optimization:} 
Recent studies have explored various optimization techniques.
SS-DBSCAN~\cite{monoko2023optimizedparam} and AMD-DBSCAN~\cite{Wang2022AMDDBSCANAA} both suggest an exhaustive grid-search approach for $MinPts$, where the former includes an automated version of the Elbow method using stratified sampling. AEDBSCAN~\cite{Mistry2021AEDBSCAN} assigns a per-point radius to optimize $\varepsilon$ for a fixed $MinPts$.
These algorithms essentially apply exhaustive search of the optimal parameters without relying on the underlying algorithm properties and their results are reported over low dimensional synthetic datasets.
We consider the following density-based parameter tuning methods as the SOTA baselines \cite{campello2013density,liu2007vdbscan,ankerst1999optics,monoko2023optimizedparam,Wang2022AMDDBSCANAA,Mistry2021AEDBSCAN,gaonkar2013autoepsdbscan}.

In summary, while density-based clustering methods offer advantages in handling non-linearly separable data, their reliance on parameter settings poses a significant challenge, especially in high-dimensional and large-scale scenarios. Our work builds upon these foundations, proposing an efficient method for parameter tuning that is responsive to these data representation challenges.

\section{Preliminaries} 
\label{sec:dbscan_def}
We establish standard DBSCAN algorithm definitions and notations,
\begin{definition}[Dataset]
\label{def:dataset}
Let $(M,d)$ be a metric space and a distance metric. A dataset
$X = \{x_1, \ldots, x_N\} \subseteq M$ is a finite subset of $M$.    
\end{definition}
\begin{definition}[DBSCAN Parameters]
\label{def:dbscan_parameters}
For a dataset $X \in \mathbb{R}^{N \times D}$, DBSCAN receives $\varepsilon \in \mathbb{R}_{> 0}$ and $MinPts \in \mathbb{N}_{\geq 2}$ as user-defined parameters.
\end{definition}

\begin{definition}[$\varepsilon$-ball and Neighborhood]
For a point $p \in \mathbb{R}^D$, $X$, and $\varepsilon > 0$:
\begin{align*}
\text{Let}\quad B_\varepsilon(p) &= \{y \in \mathbb{R}^D : d(p,y) \leq \varepsilon\}\quad \text{be the}~\varepsilon\text{-ball centered at}~p\\ 
\text{and let}\quad N_\varepsilon(p) &= B_\varepsilon(p) \cap X \quad \text{be the}~\varepsilon\text{-neighborhood of}~p
\end{align*}
\end{definition}
\begin{definition}[Core and Border Points]
For $\varepsilon$ and $MinPts$:
\begin{itemize}
    \item A point $p \in X$ is a \textit{core point} if $|N_\varepsilon(p)| \geq MinPts$
    \item A point $b \in X$ is a \textit{border point} if:
        \begin{itemize}
            \item $b \in B_\varepsilon(p)$ for some core point $p$
            \item $|N_\varepsilon(p)| < MinPts$ (not itself a core point)
        \end{itemize}
    \item $p$ is a \textit{noise point} if it is neither a core nor a border point
\end{itemize}
\end{definition}
\begin{definition}[Density-Reachability]
Two points $p,q \in X$ are density-reachable at radius $\varepsilon$ if there exists a sequence of points $\{p_1, \ldots, p_t\} \subseteq X$ such that:
\begin{enumerate}
    \item $p_1 = p$ and $p_t = q$
    \item Each $p_i$ is a core point for $i < t$
    \item $p_{i+1} \in B_\varepsilon(p_i)$ for all $i < t$
\end{enumerate}
\end{definition}
\begin{definition}[Cluster]
A cluster $c$ is a maximal set of points where:
\begin{itemize}
    \item At least one point in $c$ is a core point
    \item Non-core points in $c$ are within $\varepsilon$-radius of a core point in $c$
    \item All core points in $c$ are mutually density-reachable
\end{itemize}
\end{definition}

DBSCAN forms clusters by identifying core points and their density-reachable neighbors within $\varepsilon$-radii. Non-core points that are density-reachable from core points become border points, and the remaining points are classified as noise~\cite{schubert2017dbscan}.

Let $\mathbf{A}:\mathbb{R}^{N\times D}\times\mathbb{R}\times\mathbb{N}~\to~\mathbb{N}^N$ denote DBSCAN or a variant of DBSCAN, let $\mathcal{C}:=\mathbf{A}(X,\varepsilon,MinPts)\in \mathbb{N}^N$ be $\mathbf{A}$'s output clustering assignment, and let 
$\mathbf{K}(\mathcal{C})=|\{c_i\in \mathcal{C}\}_{i=1}^N|$ be the function which counts the number of clusters $\mathbf{A}$ returns.
We often refer to $\mathbf{K}(\mathcal{C})$ by using the variable $k$ when clear from context and characterize it as a function of $\varepsilon$,
\begin{equation}
    k(\varepsilon):=\mathbf{K}(\mathbf{A}(X,\varepsilon,MinPts))
    \label{eq:keps}
\end{equation}

\section{The Unimodality Property}
\label{sec:Unimodality}

We make a fundamental observation regarding DBSCAN, that for a fixed $MinPts$, $k(\varepsilon)$ (Eq.~\ref{eq:keps}) is \textit{near-unimodal}. This is since
\textbf{1.}~low values of $\varepsilon$ label more examples as noise so less clusters are formed (specifically for $\varepsilon<min_{i\ne j}d(x_i,x_j)$ there are no core points), while \textbf{2.}~high values of $\varepsilon$ combine clusters together and gradually reduce the number of clusters (for $\varepsilon \geq max_{i,j}d(x_i,x_j)$ there is a single cluster since all points are mutually density-reachable). 

We identify in a counter-example in Fig.~\ref{fig:nonunimodal_pathology}, that $k(\varepsilon)$ is not necessarily strictly unimodal as per the standard definition, i.e., monotonically non-decreasing up to the mode and monotonically non-increasing thereafter~\cite{hartigan1985dip}. However, we empirically demonstrate the \textit{near-unimodality} of $k(\varepsilon)$ over 24 real-life datasets, and support this by a statistical test (DIP~\cite{hartigan1985dip}) in Sec.~\ref{sec:Eval}, and a theoretical analysis. 

\begin{figure}
    \centering
    \includegraphics[width=1\linewidth]{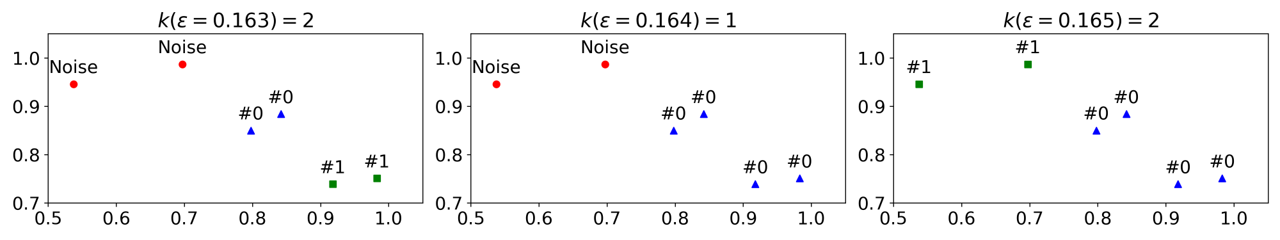}
    \caption{A counter-example with non-unimodal behavior of $k(\varepsilon)$ in $\mathbb{R}^2$ and $L_2$. }
    \label{fig:nonunimodal_pathology}
\end{figure}

\begin{figure}[htbp]
\begin{minipage}[t]{0.49\textwidth}
        \centering
        \includegraphics[width=\textwidth]{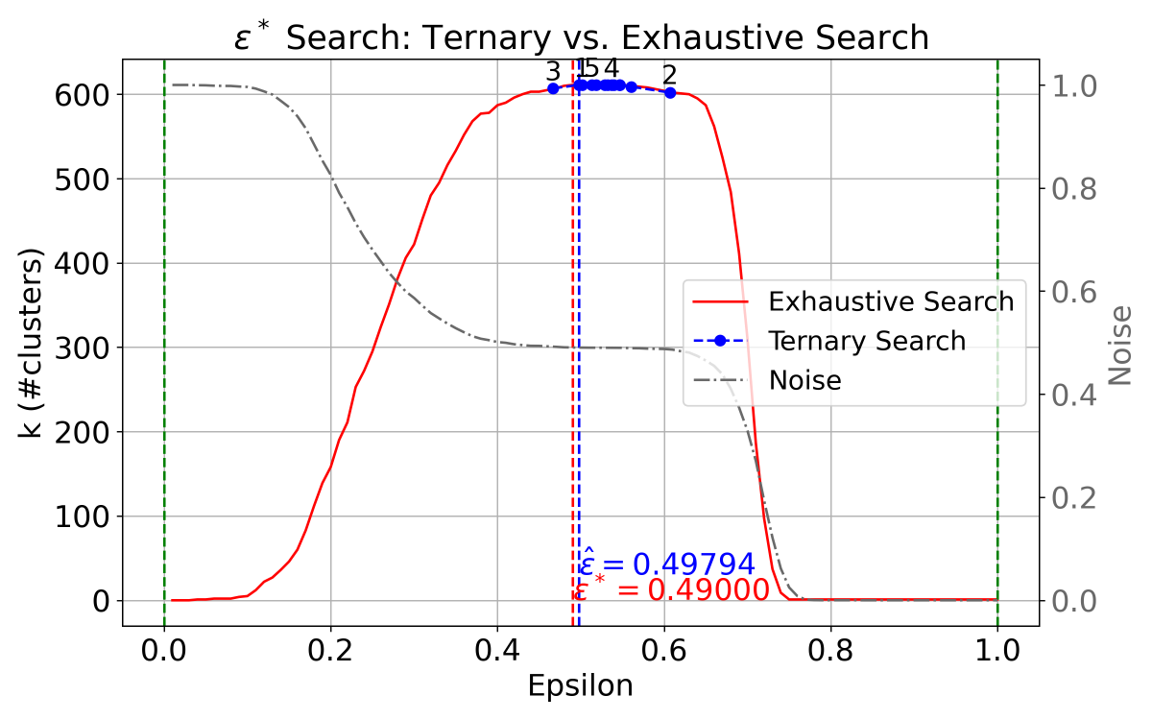}
        \caption{TS convergence steps in blue vs. Exhaustive Search in red over the FACE dataset (N=45k with noise). The grey line illustrates the percentage of noise.}
        \label{fig:ExhoustiveVsTernary}
\end{minipage}
\hfill
\begin{minipage}[t]{0.49\textwidth}
    \centering
        \centering
        \includegraphics[width=\textwidth]{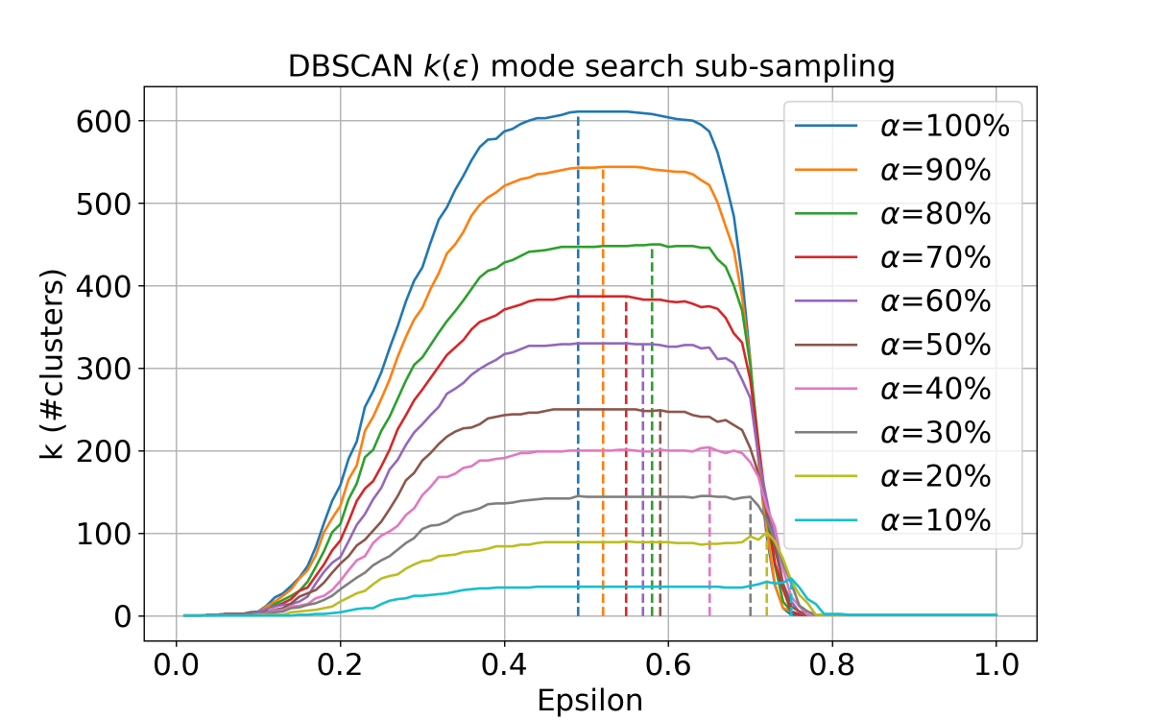}
        \caption{Exhaustive search of $k(\varepsilon)$ unimodality for different sample sizes ($\alpha$). When $\alpha$ grows $\varepsilon^*_\alpha$ shrinks monotonically, hence used as an $UB$ in our method.}
        \label{fig:AllAlphas}
\end{minipage}
\end{figure}

\subsection{Theoretical Analysis:}
\label{subsec:TheoreticalAnalysis}
To support our unimodality claim, we give a theoretical analysis of DBSCAN running on a dataset $X=\{x_1,\ldots,x_N\}$ of iid uniform samples $\forall i, x_i\sim U[0,1]^D$. We acknowledge this distribution is non-standard for clustering, but it is a first step towards understanding this property in more general distributions.
First, for $MinPts=2$, a common choice for this parameter, and $D=1$. Theorem~\ref{thm:minpts2} gives and proves the explicit unimodal function describing $E_X[k(\varepsilon)]$, and gives its mode. We note that in experimental evaluations which are omitted, we saw that a law of large numbers appears, and for $N>1,000$, even for any single random dataset $X\sim U[[0,1]^N]$, almost always the resulting $k(\varepsilon)$ was surprisingly close to the unimodal function $E_X[k(\varepsilon)]$ over the whole domain $\varepsilon>0$, hence near-unimodal itself.
Then, we consider the setting where $N$ is large, and assume that $MinPts$ grows as a function of $N$ (this is standard in DBSCAN for large datasets). For ease of presentation, we assume a constant ratio $\rho=MinPts/N \in (0,1)$ between $MinPts$ and $N$. This setting is natural for DBSCAN since doubling the dataset size means that the expected number of points in each $\varepsilon$-ball should roughly double. To warm up, Theorem~\ref{thm:minptsrho} considers $D=1$, and then Theorem~\ref{thm:minptsrhogenerald} considers a general $D\in \mathbb{N}$, and they both prove that $k(\varepsilon)$ is not trivially 0 or 1 only for $\varepsilon\approx\frac{1}{2}\rho^\frac{1}{D}$ (much stronger than the bounds to come in Section~\ref{sec:ternarySearch}). We stress that while we present our results for a constant ratio $\rho$, they in fact hold for a sufficiently large but reasonable $N$ for any setting where $MinPts$ grows (e.g., logarithmically) with $N$, and linearly with $D$ (a common assumption is that $MinPts>D$).
This suggests that when we search for $\varepsilon^*$, first identifying the often small region in which $k(\cdot)$ is not trivial, can improve runtime and accuracy.

\begin{theorem}\label{thm:minpts2}
Consider DBSCAN running on a uniformly sampled dataset $X\sim U[[0,1]^N]$ with parameter $MinPts=2$. The expected number of clusters as a function of $\varepsilon>0$ is $E_{X}[k(\varepsilon)]=(N-1)\cdot (1-\varepsilon)_+^N-(N-2)\cdot (1-2\varepsilon)_+^N$ where $y_+:=\max(y,0)$. This is a unimodal function maximized at $\varepsilon_0=\frac{a-1}{2a-1}\approx_{N\to\infty}\frac{\ln(2)}{N}$ for $a:=(\frac{2N-2}{N-2})^\frac{1}{N-1}$, with expected value $E_{X}[k(\varepsilon_0)]\approx_{N\to\infty} \frac{N}{4}$.
\end{theorem}
\begin{proof}[Proof Idea]
    Consider the order statistics $u_1\leq \ldots\leq u_N$ of $X$, and the differences, often called \textit{spacings}, $s_i:=u_i-u_{i-1}$ (by convention $u_0:=0$, $u_{N+1}:=1$). Let $U_i$ be  the event that $u_i$ is the rightmost point of a cluster. Note that for $1<i\leq N-1$, $U_i$  occurs iff $u_i$ is a core point and the cluster does not extend to the right, i.e., $U_i=(s_i<\varepsilon) \wedge (s_{i+1}>\varepsilon)$. This occurs w.p. $$P_X[U_i] = P_X[(s_i<\varepsilon) \wedge (s_{i+1}>\varepsilon)]=P_{X}[s_{i+1}>\varepsilon]-P_{X}[s_i>\varepsilon\wedge s_{i+1}>\varepsilon].$$
    For spacings, it is well known~\cite{devroye1986uniform} that for any nonnegative values $b_1,\ldots,b_N$ it holds that, $P[\forall i, s_i>b_i]=(1-\Sigma_{i}b_i)_+^{N}$, so we conclude $P_X[U_i]= (1-\varepsilon)_+^N-(1-2\varepsilon)_+^N$. 
    Note that $P[U_1]=0$ since the leftmost point cannot have a left neighbor, and that $P_X[U_{N}]=P_X[s_N<\varepsilon]=(1-\varepsilon)_+^N$ because there is no point to its right.
    Since $E[A]=P[A]$ for a Bernoulli RV $A$, and by linearity of the expectation, since the number of clusters is the number of rightmost points of clusters, $$E_{X}[k(\varepsilon)]=E_{X}[\Sigma_i \mathbbm{1}_{U_i}]=\Sigma_i P_{X}[U_i]= (N-1)\cdot (1-\varepsilon)_+^N-(N-2)\cdot (1-2\varepsilon)_+^N.$$
    To find the optimal $\varepsilon$, note that it is attained at $\varepsilon<1/2$ and solve $\frac{d}{d\varepsilon}E_X[k(\varepsilon)]=0$ for $\varepsilon$, i.e., $\frac{d}{d\varepsilon}E_X[k(\varepsilon)]=2N(N-1)(1-2\varepsilon)^{N-1}-N(N-2)(1-\varepsilon)^{N-1}=0$.\qed
\end{proof}
We prove the concentration bound,
\begin{theorem}\label{thm:minptsrho}
    Consider DBSCAN running on a uniformly sampled dataset $X\sim U[[0,1]^N]$, with parameter $MinPts=\rho\cdot N$ for a constant $\rho<\frac{1}{4}$. Then for any $\beta>1$ and $\delta>0$, both conditions hold: \textbf{1.} for any $\varepsilon>\beta\cdot \frac{\rho}{2}$ and $N>N_\varepsilon$, $P_{X}[k(\varepsilon)=1]>1-\delta$, and \textbf{2.} for any $\varepsilon<\frac{1}{\beta}\cdot \frac{\rho}{2}$ and $N>N_\varepsilon$, $P_{X}[k(\varepsilon)=0]>1-\delta$,\\ for an appropriately large  $N_\varepsilon=\Omega\left(\log(\frac{1}{\delta\varepsilon q})/(\rho q^2)\right)$ for $q:=1-\frac{1}{\sqrt{\beta}}$.
\end{theorem}
\begin{proof}[Proof Idea]
    For the case $\varepsilon>\beta\cdot \frac{\rho}{2}$, divide $[0,1]$ to the set of segments $S_l=[al-b,al+b],~\text{for all }l\in \mathbb{N}$ where $b=\frac{\varepsilon}{\sqrt{\beta}}$ is slightly below $\varepsilon$, and $a=\varepsilon-b$. By the Hoeffding inequality and a union bound, w.h.p. all segments $S_l$ centered in $[\varepsilon,1-\varepsilon]$ have $\geq MinPts$ points, so for each $x_i\in [\varepsilon,1-\varepsilon]$, the $\varepsilon$-ball around it  contains such a segment (for some $l\in \mathbb{N}$, by the definition of $a$ and the triangle inequality), hence $x_i$ is a core point. Combined with a high probability event that for large enough $N$, all points are $\varepsilon$-near each other, and that the cluster covers also $[0,\varepsilon]$ and $[1-\varepsilon,1]$, we get a single cluster.
 
    For the case $\varepsilon<\frac{1}{\beta}\cdot \frac{\rho}{2}$, divide to segments as above with $b=\varepsilon\cdot\sqrt{\beta}$ slightly above $\varepsilon$, and $a=b-\varepsilon$. By Hoeffding inequality and a union bound, w.h.p. all segments $S_l$ centered in $[\varepsilon,1-\varepsilon]$ have $< MinPts$ points, so since the $\varepsilon$-neighborhood of each point $x_i$ is contained in such a segment (for some $l\in \mathbb{N}$, by triangle inequality and $a=b-\varepsilon$), there are no core points hence no clusters.\qed
\end{proof}
We follow similar arguments to those in Theorem~\ref{thm:minptsrho}, and extend it to uniform datasets on the unit cube $[0,1]^D$ of arbitrary dimension $D\in \mathbb{N}$, incurring only an additional $\sqrt{D}$ factor on the upper bound and no additional factor on the lower bound.
\begin{theorem}\label{thm:minptsrhogenerald}
    Consider DBSCAN running on a dataset $X\sim U[[0,1]^{N\times D}]$ of $N$ uniformly sampled points in the unit $D$-dimensional cube, with parameter $MinPts=\rho\cdot N$ for a constant $\rho<\frac{1}{4}$. Then for any $\beta>1$ and $\delta>0$, both conditions hold: \textbf{1.} for any $\varepsilon>\sqrt{D}\beta\cdot \frac{1}{2}\rho^{\frac{1}{D}}$ and $N>N_\varepsilon$, $P_{X}[k(\varepsilon)=1]>1-\delta$, and \textbf{2.} for any $\varepsilon<\frac{1}{\beta}\cdot \frac{1}{2}\rho^{\frac{1}{D}}$ and $N>N_\varepsilon$, $P_{X}[k(\varepsilon)=0]>1-\delta$,\\ for an appropriately large\\  $N_\varepsilon=\Omega\left(\frac{1}{\rho}\cdot \left[\frac{\sqrt{\beta}}{(\frac{1}{\sqrt{\beta}}-1)^2}\right]\cdot \left[D\log(\frac{2}{\beta-\sqrt{\beta}})+\log(\frac{1}{\rho\alpha})\right]\right)$.
\end{theorem}
\begin{proof}[Proof Idea]
    For the case $\varepsilon>\sqrt{D}\beta\cdot \frac{1}{2}\rho^{\frac{1}{D}}$, divide $[0,1]^D$ to the set of cubes $S_{l_1,\ldots,l_D}=[al_1-b,al_1+b]\times\ldots\times [al_D-b,al_D+b],~\text{for all }l_1,\ldots,l_D\in \mathbb{N}$ where $b=\varepsilon/(\sqrt{D}\beta^{\frac{1}{2D}})$ is slightly below $\varepsilon/\sqrt{D}$, and $a=\varepsilon/\sqrt{D}-b$. By the Hoeffding inequality and a union bound, w.h.p. all cubes $S_{l_1,\ldots,l_D}$ centered in the set $[\varepsilon,1-\varepsilon]^D\subset [0,1]^D$ have $\geq MinPts$ points, so for each $x_i\in [\varepsilon,1-\varepsilon]^D$, the $\varepsilon$-ball around it  contains such a cube (for some $l\in \mathbb{N}$, by the definition of $a$ and the triangle inequality, and that the diameter of such cubes are $=\sqrt{D}b$ slightly below $\varepsilon$), hence $x_i$ is a core point. Combined with a high probability event that for large enough $N$, all points are $\varepsilon$-near each other, and that the cluster covers also the exterior $[0,1]^D\setminus[\varepsilon,1-\varepsilon]^D$, we get a single cluster.
 
    For the case $\varepsilon<\frac{1}{\beta}\cdot \frac{1}{2}\rho^{\frac{1}{D}}$, divide to cubes as above with $b=\varepsilon\cdot \beta^{\frac{1}{2D}}$ slightly above $\varepsilon$, and $a=b-\varepsilon$. By Hoeffding inequality and a union bound, w.h.p. all cubes $S_{l_1,\ldots,l_D}$ centered in the set $[\varepsilon,1-\varepsilon]^D$ have $< MinPts$ points, so since the $\varepsilon$-neighborhood of each point $x_i$ is contained in such a cube (for some $l\in \mathbb{N}$, by triangle inequality and $a=b-\varepsilon$), there are no core points hence no clusters.\qed
\end{proof}

These theoretical results demonstrate that, under natural assumptions, to find the value of $\varepsilon$ that maximize $k(\varepsilon)$, we can first focus on a small range where $k(\varepsilon)$ is neither trivially 1 nor 0, and subsequently exploit its \textit{near-unimodal} behavior to efficiently find its mode, as discussed in the next section.

\newpage

\section{Method}
\label{sec:Method}
Our task is to efficiently find the mode of $k(\varepsilon)$, 
\begin{equation}
    \label{eq:eps_star}
    \varepsilon^* \mathrel{=} \textit{argmax}_{\varepsilon} \{ \mathbf{K}(\mathbf{A}(X,\varepsilon,MinPts)) \}
\end{equation}
The Ternary Search (TS) algorithm by Bajwa et al.~\cite{bajwa2015ternary} finds a maximum in a unimodal (discrete) array. In Section|\ref{sec:ternarySearch}, we adapt it to functions, leveraging the \textit{near unimodality} of $k(\varepsilon)$, to find $\varepsilon^*$ using fewer evaluations of $k(\cdot)$ compared to a linear search (see Fig.~\ref{fig:ExhoustiveVsTernary}).
Then, in Section~\ref{sec:tse} we introduce an even quicker estimator (TSE) for $\varepsilon^*$.

\subsection{Ternary Search for $\varepsilon^*$ (TS)}\label{sec:ternarySearch}
Our TS algorithm (Alg.~\ref{alg:TernarySearch}) leverages the fact that the function $k(\varepsilon)$ is near-unimodal with mode $\varepsilon^*$. It starts with an initial lower bound ($LB$) and upper bound ($UB$) for $\varepsilon^*$, and iteratively ($itr$ times) divides it to 3 equal parts and removes at least one of them (see Alg.~\ref{alg:TernarySearchConditions}), to reduce the search space size and still contain the mode. Specifically, for Alg.~\ref{alg:TernarySearchConditions}, we let $m_l=\frac{2LB+UB}{3}$ and $m_r=\frac{LB+2UB}{3}$, and $k_l=k(m_l)$ and $k_r=k(m_r)$ be the respective cluster counts. Recall that the cluster count as a function of $\varepsilon$ is initially 0, then increases until it reaches the mode, and then decreases back to 1. Hence, the space reduction logic is as follows: 1. If $k_l=k_r=1$, the mode must be below them, 2. If $k_l=k_r=0$, the mode must be above them, 3. if $k_l=0$ and $k_r=1$, then the mode is between them, 4. if $k_l>k_r$ then the mode is to the left of $m_r$ and otherwise the mode is to the right of $m_l$.
The edge-case in which $k=1$ for the very first formed cluster i.e., by chance and not by convergence, is easily detected by the Noise ratio thus omitted for simplicity.
We provide both algorithms pseudo-code below.
\begin{figure}
    \noindent
    \begin{minipage}[t]{0.49\textwidth}
    \begin{algorithm}[H]
    {
        \begin{algorithmic}[1]
        \caption[]{$TS(X,LB,UB,MinPts,itr)$}
        \label{alg:TernarySearch}
        \FOR{i=0 to itr} 
            \STATE $m_l \leftarrow \frac{2LB + UB}{3}$
            \STATE $m_r \leftarrow \frac{LB + 2UB}{3}$
            \STATE $\mathcal{C}_l \leftarrow \mathbf{A}(X,m_l,MinPts)$
            \STATE $\mathcal{C}_r \leftarrow \mathbf{A}(X,m_r,MinPts)$
            \STATE $\langle LB, UB\rangle \gets Cond(LB,UB,\dots$
            \STATE $\quad\quad\quad\quad~~~~~~~~~~m_l,m_r,\mathbf{K}(\mathcal{C}_l),\mathbf{K}(\mathcal{C}_r))$
        \ENDFOR
        \RETURN $\frac{m_l+m_r}{2}$
        \end{algorithmic}
    }
    \end{algorithm}
        \end{minipage}
    \hfill
        \begin{minipage}[t]{0.49\linewidth}
    \begin{algorithm}[H]
    {
        \begin{algorithmic}[1]
        \caption[]{$Cond(LB,UB,m_l,m_r,k_l,k_r)$}
        \label{alg:TernarySearchConditions}
            \IF{ $k_l==1 ~ \AND ~ k_r==1$}
                \RETURN $\langle LB,m_l\rangle$
            \ELSIF{$ k_l==0 ~ \AND ~ k_r==1$}
                \RETURN $\langle m_l,m_r\rangle$
            \ELSIF{$ k_l==0 ~ \AND ~ k_r==0$}
                \RETURN $\langle m_r,UB\rangle$
            \ELSIF{$k_l>k_r$}
                \RETURN $\langle LB,m_r\rangle$
            \ELSE
                \RETURN $\langle m_l,UB\rangle$
            \ENDIF
        \end{algorithmic}
    }
    \end{algorithm}
    \end{minipage}%
\end{figure}

\textbf{Upper and Lower Bounds:}
To efficiently initialize the search space, we use tight bounds $UB$, $LB$ for $\varepsilon^*$. For bounded metrics, a trivial upper bound $UB^0$ is simply their bound, and for (unbounded) metrics in general, as in Sec.~\ref{sec:Unimodality}, an upper bound is the Diameter $\max_{i,j \in \{1,\dots,N\}}d(x_i, x_j)$ of $X$, which we approximate in linear time by doubling a 2-approximation of it. A trivial lower bound $LB^0$ is 0. We provide an improved heuristic for $UB$ and $LB$ via sampling:
\begin{itemize}
\item \textit{Upper Bound ($UB$):} 
We observe that empirically, as depicted in Fig.~\ref{fig:AllAlphas}, sub-sampling an  $\alpha$ fraction of the data requires a larger radius to form a core-point since data is more sparse, so  $\varepsilon^*$ increases. 
Hence, we produce the upper bound
\begin{equation}
    \label{eq:upper_bound}
    UB\mathcal{=}TS(X_{\mathcal{R},1:D},LB^0,UB^0,MinPts,itr)
\end{equation}
\begin{equation}
    \text{for }\mathcal{R} \sim \text{Uniform}(G \subset \{1, \ldots, N\}:|G|=\lceil \alpha N \rceil)\nonumber
\end{equation}

\item \textit{Lower Bound ($LB$):} Empirically, projecting the data on a random subset of $\alpha D$ dimensions, reduces $\varepsilon^*$ since it brings the data closer. Hence, define
\begin{equation}
    \label{eq:lower_bound}
    LB\mathcal{=}TS(X_{1:N,\mathcal{T}},LB^0,UB,MinPts,itr)
\end{equation}
\begin{equation}
    \text{for }\mathcal{T} \sim \text{Uniform}(H \subset \{1, \ldots, D\}:|H|=\lceil \alpha D \rceil)\nonumber
\end{equation}
\end{itemize}

Our resulting clustering algorithm TSClustering (Alg.~\ref{alg:Main}) invokes TS 3 times to find $UB,LB,$ and $\varepsilon^*$.
 \begin{algorithm}[H]
        \begin{algorithmic}[1]
        \caption[]{TSClustering($X, MinPts, itr$)}
        \label{alg:Main}
        \STATE $UB^0 \gets 2\textit{max}_{i \in \{2, \ldots, N\}} d(X_1, X_i)$
        \STATE $LB^0\leftarrow 0$
        \STATE $\mathcal{R} \leftarrow \text{sample}~\lceil \alpha N \rceil~\text{points from $X$}$
        \STATE $\mathcal{T} \leftarrow \text{sample}~\lceil \alpha D \rceil~\text{dimensions from $X$}$
        \STATE $UB \leftarrow TS(X_{\mathcal{R},1:D},LB^0,UB^0,MinPts,itr)$
        \STATE $LB \leftarrow TS(X_{1:N,\mathcal{T}},LB^0,UB,MinPts,itr)$
        \STATE $\varepsilon^* \gets TS(X, LB, UB, MinPts, itr)$
        \RETURN $A(X,\varepsilon^*,MinPts)$
        \end{algorithmic}
\end{algorithm}

\textbf{Runtime Analysis:}
TS executes $itr$ iterations, each invoking DBSCAN twice, resulting in $O(itr\cdot D N^2)$. With $itr$ empirically set to 6, the overall complexity remains $O(DN^2)$, improving upon prior works' computational efficiency.

\subsection{Ternary Search Estimator (TSE)}\label{sec:tse}
To further optimize the runtime of the heavy part of Alg.~\ref{alg:Main} (line 7) which uses the whole data and full dimension, we propose an estimator ($TSE$) for $\varepsilon^*$ obtained by sampling an $\alpha$ fraction of the data and dimensions simultaneously.
The intuition is that the opposite influences of sampling the data and dimensions on $\varepsilon^*$ should roughly cancel out. Formally, we produce the estimate by replacing line 7 with $TS(X',LB,UB,MinPts,itr)$, where $X'=X_{\mathcal{R},\mathcal{T}}$ and $\mathcal{R},\mathcal{T}$ are sampled as in Eq.~\ref{eq:upper_bound},\ref{eq:lower_bound}. To reduce the variance, we repeat the above $m$ times ($m=30$ in our experiments) and average the estimates.

\section{Evaluation}
\label{sec:Eval}
We evaluated the unimodality property and our methods across various domains, tasks, datasets, and representation techniques, using both qualitative and quantitative analyses over 24 datasets (Tab.~\ref{tab:unimodality_eval}).
To evaluate clustering with noise, we employed the standard Normalized Mutual Information (NMI), Adjusted Rand Index (ARI), and Noise~\cite{monoko2023optimizedparam,schubert2017dbscan,ester1996density} metrics.
NMI normalizes by cluster entropy, suitable for varying cluster counts, while ARI provides chance-adjusted accuracy with error-balance penalty. Together, these metrics offer complementary perspectives~\cite{manning2008introduction,steinley2004properties}.

\begin{figure*}[t]
    \centering
    \begin{minipage}[b]{0.48\linewidth}
        \centering
        \includegraphics[width=1.\linewidth]{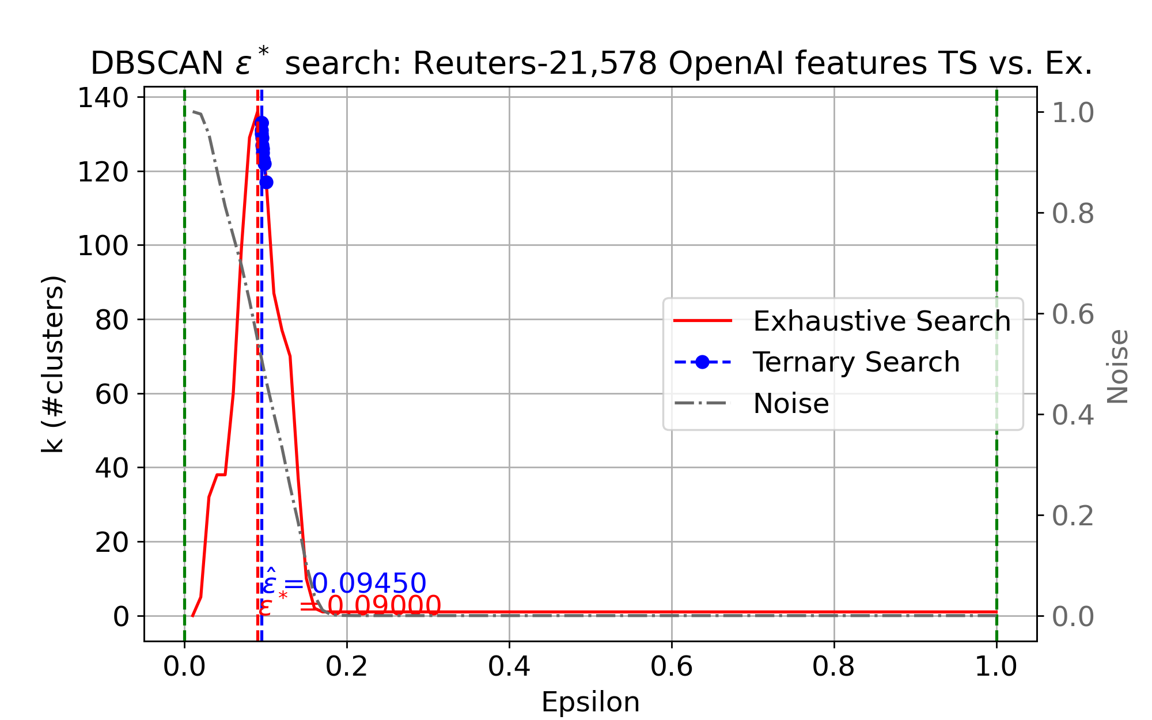}
        \subcaption{Reuters}
    \end{minipage}
    \begin{minipage}[b]{0.48\linewidth}
        \centering
        \includegraphics[width=1.\linewidth]{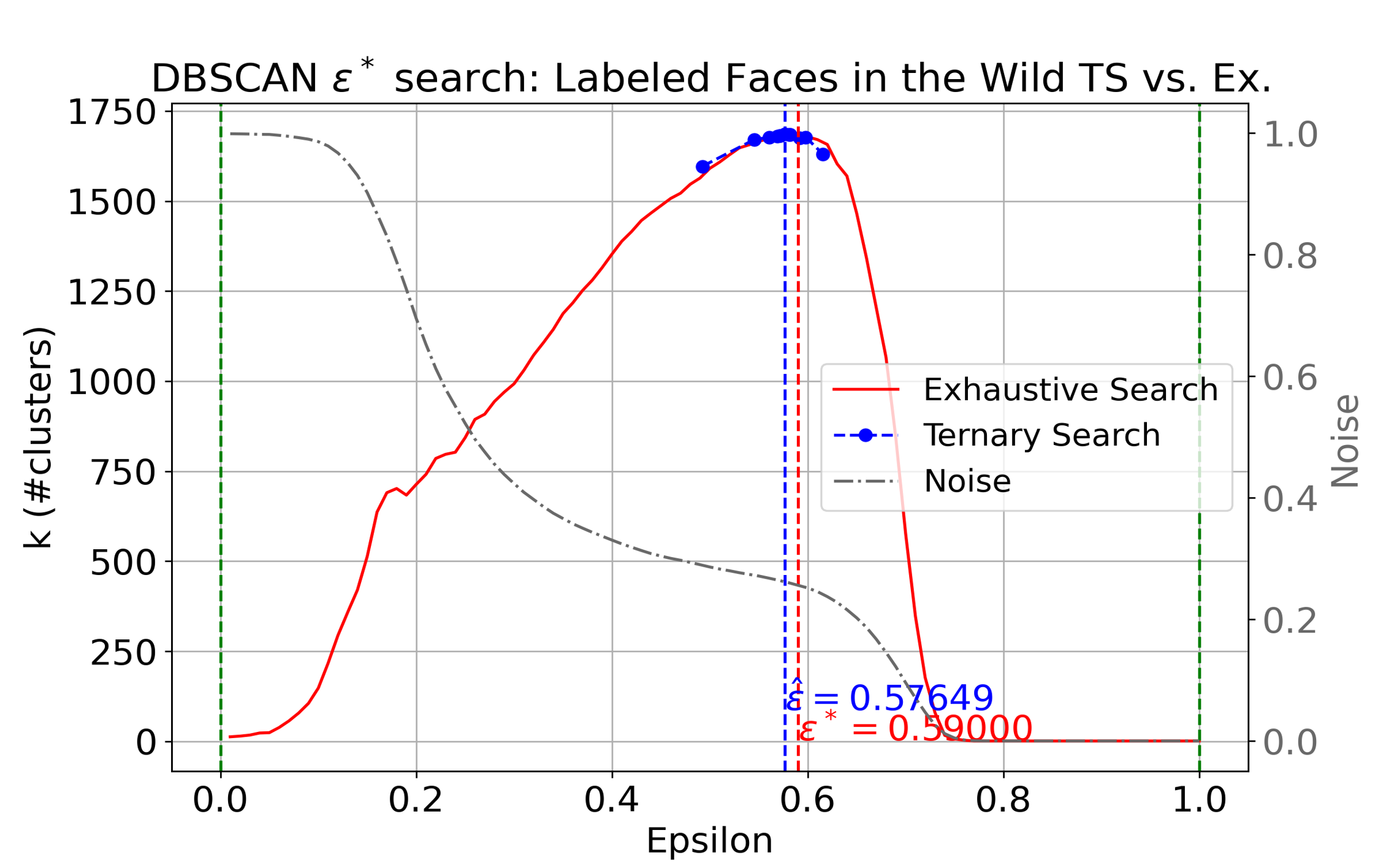}
        \subcaption{LFW}
    \end{minipage}   
        \begin{minipage}[b]{0.48\linewidth}
        \centering
        \includegraphics[width=1.\linewidth]{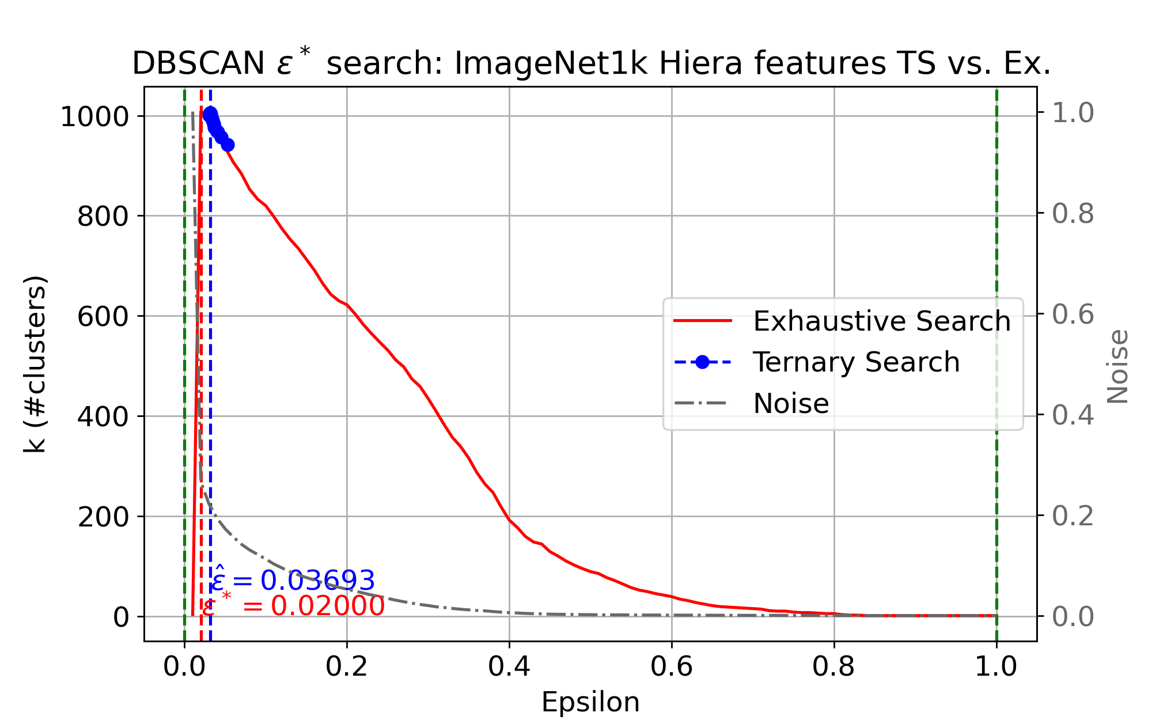}
        \subcaption{ImageNet1k}
    \end{minipage}
    \begin{minipage}[b]{0.48\linewidth}
        \centering
        \includegraphics[width=1.\linewidth]{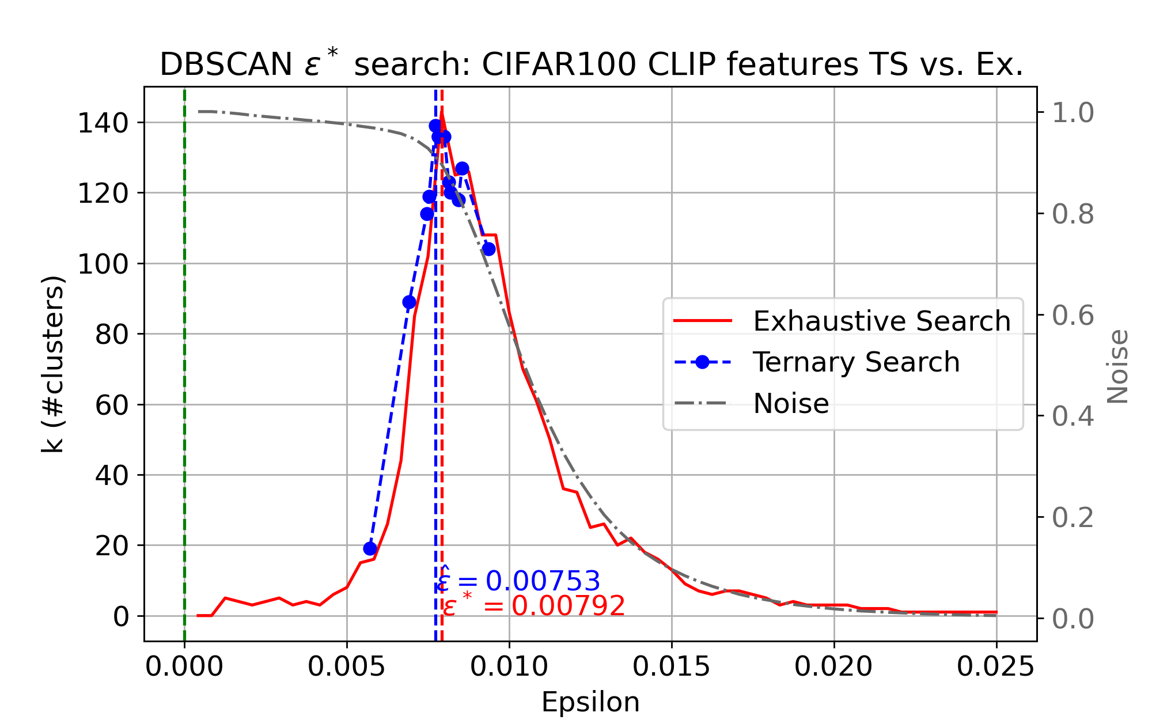}
        \subcaption{CIFAR100}
    \end{minipage}
    \begin{minipage}[b]{0.48\linewidth}
        \centering
        \includegraphics[width=1.\linewidth]{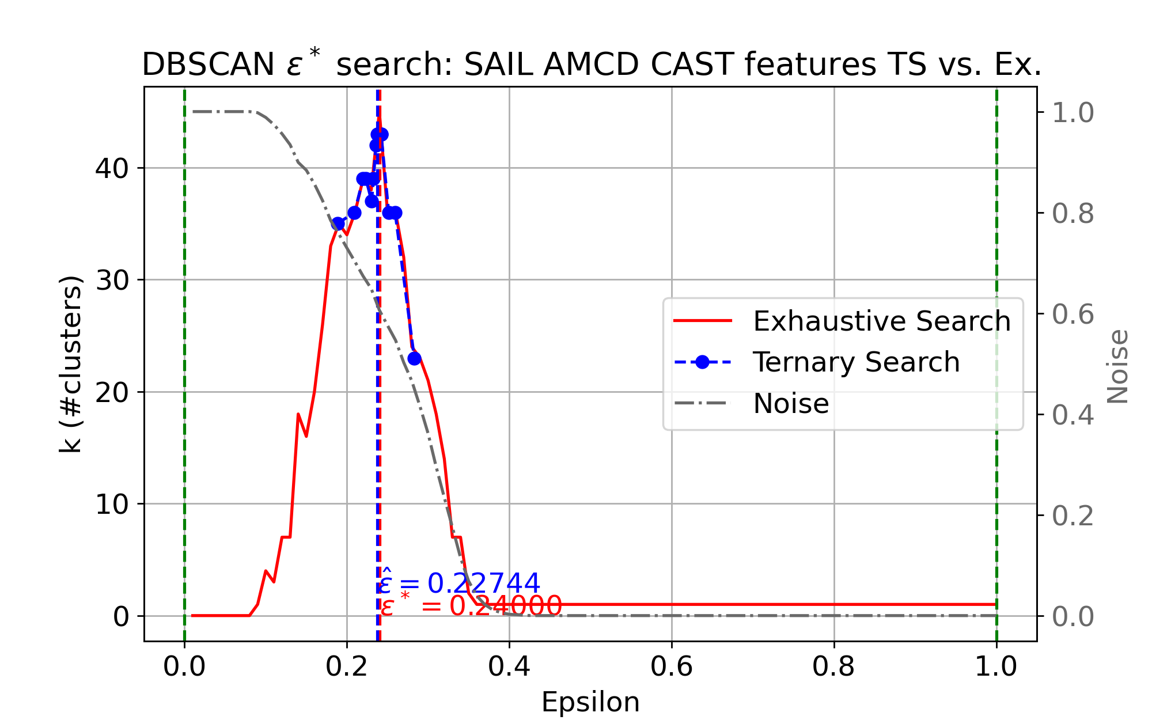}
        \subcaption{SAIL AMCD}
    \end{minipage}
        \begin{minipage}[b]{0.48\linewidth}
        \centering
        \includegraphics[width=1.\linewidth]{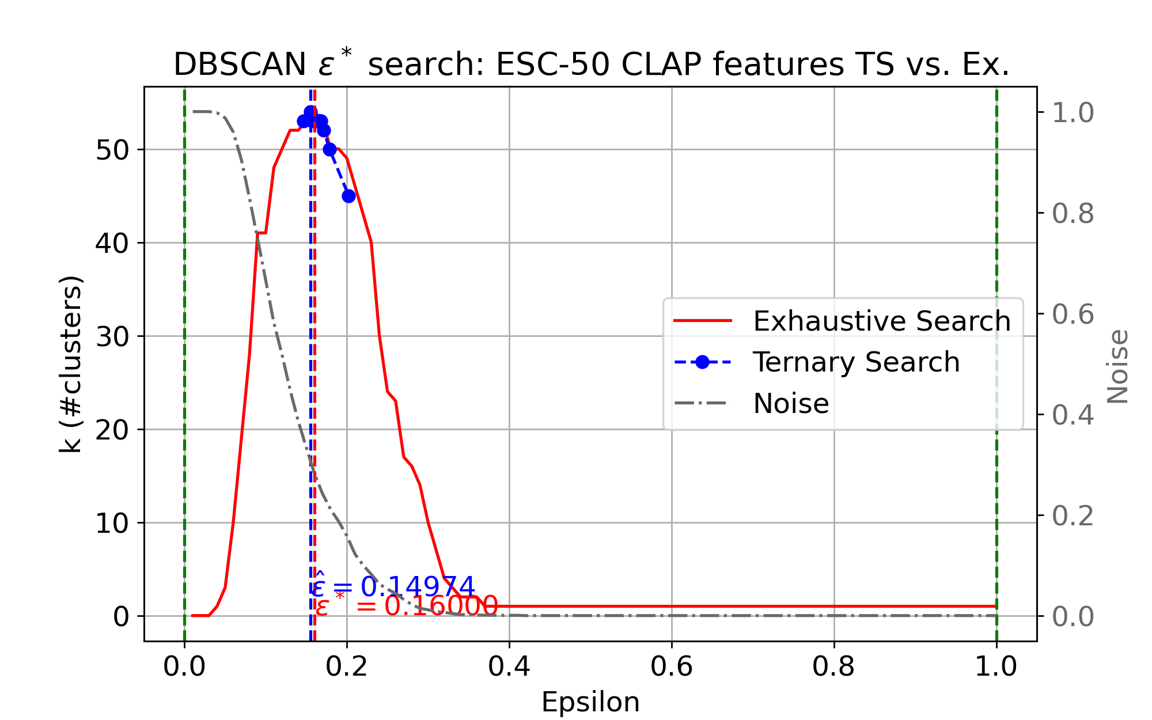}
        \subcaption{ESC-50}
    \end{minipage}
    \caption{Datasets across Vision, Audio, and NLP demonstrate unimodality of $k(\varepsilon)$.}
    \label{fig:QualitativeUimodality2}
\end{figure*}

\textbf{Unimodality Analysis:}
Qualitatively,
Fig.~\ref{fig:QualitativeUimodality2} in the appendix demonstrates the unimodal shape (in red) of $k(\varepsilon)$ across 6 datasets from various fields. We used evenly spaced samples of $\varepsilon$ for these curves, and executed DBSCAN exhaustively for each value. See Fig.~\ref{fig:CASTSAILUnimodal} in the appendix for 15 additional datasets.

\begin{table}[ht]
\caption{DIP Test over datasets~\cite{huang2008labeled,ImageNet1k2009,krizhevsky2009learning,lewis1997reuters,piczak2015dataset,somandepalli2017unsupervised,nir2022cast} and embeddings~\cite{nir2022cast,huang2018densely,ryali2023hiera,radford2021learning,radford2019language,elizalde2024natural}.}
    \label{tab:unimodality_eval}
  \centering
  \resizebox{\columnwidth}{!}{
    \begin{tabular}{l|ccccc|c||l|ccccc|c}
    \toprule
    Dataset   & Labels & $N$  & Embed.     & $D$    & Task      & $p^{DIP}_{val}$ & Dataset   & Labels & $N$  & Embed.     & $D$    & Task      & $p^{DIP}_{val}$\\
    \midrule 
     LFW      & 1,680 & 13,233 & DNet  & 256    & Face  & >99.9\% & AMCDv5  & N/A   & 13,406 & CAST & 2,048  & Anim  & 14.9\% \\
     ImNet1k  & 1,000 & 50,000 & CLIP  & 512    & OD    & >99.9\% & AMCDv6  & N/A   & 14,372 & CAST & 2,048  & Anim  & 6.4\%  \\
     ImNet1k  & 1,000 & 50,000 & Hiera & 1,000  & OD    & 33.8\%  & AMCDv7  & N/A   & 14,460 & CAST & 2,048  & Anim  & 8.4\%  \\
     CIFAR    & 100   & 60,000 & CLIP  & 512    & OD    & 8.9\%   & AMCDv8  & N/A   & 14,748 & CAST & 2,048  & Anim  & 14.8\% \\
     CIFAR    & 100   & 60,000 & Hiera & 1,000  & OD    & >99.9\% & CASTv1  & N/A   &  2,648 & CAST & 2,048  & Anim  & 6.4\%  \\ 
     Reuters  & 135   & 21,578 & ADA2  & 1,536  & Doc   & 99.8\%  & CASTv2  & N/A   &  4,215 & CAST & 2,048  & Anim  & 79.1\% \\ 
     ESC-50   & 50    &  1,024 & CLAP  & 1,024  & Audio & 41.3\%  & CASTv3  & N/A   &  4,633 & CAST & 2,048  & Anim  & 14.8\% \\
     FACE     & N/A   & 45,207 & DNet  &   256  & Face  & >99.9\% & CASTv4  & N/A   &  4,163 & CAST & 2,048  & Anim  & 99.4\% \\ 
     AMCDv1   & N/A   & 15,395 & CAST  & 2,048  & Anim  & 29.3\%  & CASTv5  & N/A   &  4,959 & CAST & 2,048  & Anim  & 14.8\% \\ 
     AMCDv2   & N/A   & 13,102 & CAST  & 2,048  & Anim  & 52.5\%  & CASTv6  & N/A   &  5,639 & CAST & 2,048  & Anim  & 99.4\% \\ 
     AMCDv3   & N/A   & 14,676 & CAST  & 2,048  & Anim  & 79.0\%  & CASTv7  & N/A   &  4,795 & CAST & 2,048  & Anim  & 52.5\% \\ 
     AMCDv4   & N/A   & 14,676 & CAST  & 2,048  & Anim  & 29.4\%  & Urban8k & N/A   &  8,732 & CLAP & 1,024  & Audio & 99.6\% \\
     \bottomrule  
\end{tabular}
}
\end{table}
Quantitatively, to validate the unimodality of $k(\varepsilon)$ empirically we perform the DIP test~\cite{hartigan1985dip}. Its null hypothesis is that the data is unimodal, and it is rejected for $p_{val}<5\%$.
The test demonstrated strong insignificance for all 24 NLP, Vision, and Audio datasets in Tab.~\ref{tab:unimodality_eval}, i.e., $k(\varepsilon)$ is unimodal on these datasets.

\textbf{Cluster Analysis:} Prior studies predominantly focus on synthetic, low-dimensional datasets. In contrast, this work emphasizes applications in high-dimension, comparing our methods (TS and TSE) to KMeans~\cite{macqueen1967some}, HDBSCAN~\cite{campello2013density}, VDBSCAN~\cite{liu2007vdbscan}, OPTICS~\cite{ankerst1999optics}, SS-DBSCAN~\cite{monoko2023optimizedparam}, AMD-DBSCAN~\cite{Wang2022AMDDBSCANAA}, AEDBSCAN~\cite{Mistry2021AEDBSCAN}, and AutoEps~\cite{gaonkar2013autoepsdbscan}. Due to the absence of open-source implementations, we re-implemented the latter four algorithms ourselves. We find that the baseline algorithms usually struggle in high dimensions, seldom producing degenerate outputs. We test against the ground-truth labels of classification datasets from NLP (Reuters), Vision (LFW), and Audio (ESC) using the metrics $NMI$, $ARI$, $k$, Noise, and runtime (Tab.~\ref{tab:NLPTopicsEval2}). 

An ideal clustering approximates the true number of clusters $k$, which in our case is the number of classification labels. TS/TSE provided the closest $k$ approximations (roughly 5\% error), reinforcing our hypothesis that $\varepsilon^*$ reveals the natural clustering. Methods like SS-DBSCAN, which subsample data, overestimate $\varepsilon$, resulting in a single cluster. Note that KMeans requires the parameter $k$, which we set via the Elbow Method for maximal Inertia curvature~\cite{schubert2023stop}.

For $NMI$ and $ARI$, since all the algorithms assign points as noise,\footnote{For KMeans clustering, we define the noisy points to be the singleton clusters.} we excluded these noise-labeled points from the $NMI$ and $ARI$ computation to isolate the clustering quality from the noise prediction.
TS/TSE consistently achieved the best scores across datasets with $P_{val}<10^{-5}$ in Friedman non-parametric test.

Regarding noise detection, the ideal outcome identifies true noise. OPTICS underestimated $\varepsilon$, labeling nearly all points as noise, which is clearly incorrect, whereas VDBSCAN and AEDBSCAN significantly overestimated $k$. For the LFW dataset, where we defined noisy points as people with only one image ($4,069/ 13,233 \approx 30.7\%$), TS/TSE provided noise estimates closely matching this true value ($\approx$ 30.5\%).
For runtime, TSE improved on TS which was competitive with the baselines.

\begin{table*}
      \caption{Evaluation over Reuters, LFW, and ESC. We compare our methods (TS, TSE) with: KMeans (KM), HDBSCAN (HD), VDBSCAN (VD), OPTICS (OP), SS-DBSCAN (SS), AMD-DBSCAN (AM), AEDBSCAN (AE) and AutoEps (Ep). In gray, results with $Noise > 90\%$.}
      \label{tab:NLPTopicsEval2}
      \centering
      \resizebox{\linewidth}{!}{
      \begin{tabular}{l|ccccc|ccccc|ccccc}
      \toprule
      Dataset&Reuters&(k=&135)&&\faFileTextO&LFW&(k=&1,680)&&\faEye&ESC&(k=&50)&&\faVolumeUp\\
      \midrule
      Method & $NMI\uparrow$ & $ARI\uparrow$ & $k$ & $Noise\downarrow$ & T[s]$\downarrow$ & $NMI\uparrow$ & $ARI\uparrow$ & $k$ & $Noise\downarrow$ & T[s]$\downarrow$ & $NMI\uparrow$ & $ARI\uparrow$ & $k$ & $Noise\downarrow$ & T[s]$\downarrow$\\
      \midrule
      KM~\cite{macqueen1967some}         & 58.5\% & 19.9\%  & 41   & \textbf{0.0\%}    & 1,917     & 78.0\% &  78.1\%  & 773 & \textbf{0.2\%}    & 17,315   & 95.1\% & 83.3\%  & 43 & \textbf{0.0\%}  & 306 \\
      HD~\cite{campello2013density}      & 62.0\% & 2.4\%  & 1,247 & 61.4\%   & 240    & 72.1\% & 36.3\%  & 393   & 56.8\%   & 105   & 86.2\% & 44.6\%  & \textbf{52} & 17.3\%   & 8  \\
      VD~\cite{liu2007vdbscan}           & 55.4\% & 0.3\% & 2,296  & 27.5\%   & 246  & 92.3\% &  12.0\%  & 2,661  & 38.0\% & 84 & 78.7\% &  20.9\%  & 447 & 23.3\%   & 10\\
      OP~\cite{ankerst1999optics}        & \textcolor{gray}{61.3\%} & \textcolor{gray}{20.5\%} & \textcolor{gray}{37}    & \textcolor{gray}{97.1\%}   & \textcolor{gray}{505}    & 64.1\% & 24.5\%  & 390   & 65.8\%   & 202   & 56.3\% &  3.8\%  & 53 & 59.7\%   & 16 \\
      SS~\cite{monoko2023optimizedparam} & 0.0\%  & 0.0\%  & 1     & 22.9\%   & 230    & 13.5\% & 4.6\%   & 2     & 52.7\%   & 252   & 85.9\% & 46.6\%  & 43 & 16.0\%   & 9  \\
      AM~\cite{Wang2022AMDDBSCANAA}      & 41.1\%  & 28.3\%  & \textbf{134}     & 6.2\%    & 69    & 71.7\%  & 20.4\%   & 281     & 34.0\%    & \textbf{29}   & 83.4\% & 31.5\%  & 93 & 18.3\%   & 13  \\
      AE~\cite{Mistry2021AEDBSCAN}       & 49.8\% & 4.5\% & 974    & 24.6\%   & 144    & 91.8\% & 23.6\%  & 1,944    & 48.7\%   & 53  & 83.9\% & 46.7\%& 230  & 20.8\%   & 7  \\
      Ep~\cite{gaonkar2013autoepsdbscan} & 66.4\% & 67.1\%& 646    & 67.5\%   & 2,377  & 11.9\% & 1.7\%   & 56       & 57.1\%   & 82  & 91.8\% & 77.3\%& 144   & 36.0\%   & 24  \\
         \bottomrule 
      TS (ours)                                 & \textbf{77.5\%} & \textbf{93.9\%} & 138   & 55.2\%   & 152    & \textbf{99.0\%} & \textbf{96.8\%}  & 1,697 & 30.5\%   & 60    & \textbf{97.4\%} & \textbf{90.3\%}  & 57 & 32.3\%   & 5 \\
      TSE (ours)                               & \textbf{77.8\%} & 92.9\% & 150   & 38.0\%   & \textbf{24}     & \textbf{99.0\%} & \textbf{96.7\%}  & \textbf{1,694} & 30.4\%   & 41    & 96.7\% & 85.2\%  & \textbf{48} & 14.7\%   & \textbf{2} \\
      \bottomrule  
\end{tabular}
}
\end{table*}


\begin{figure*}[t]
    \centering
    \begin{minipage}[b]{0.49\linewidth}
        \includegraphics[width=\linewidth]{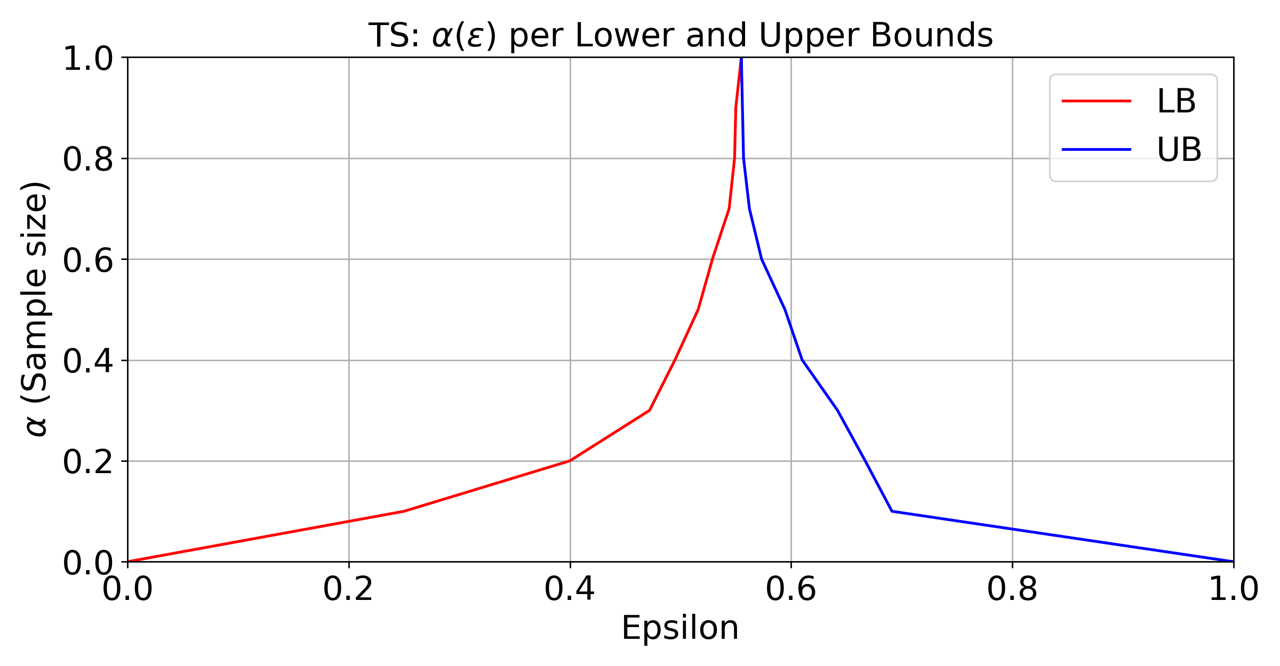}
        \subcaption{A Pareto principle: for $\alpha=20\%$ we narrow down the $\varepsilon$ range by 80\%.}
    \end{minipage}
    \hfill
    \centering
    \begin{minipage}[b]{0.49\linewidth}
        \centering
        \includegraphics[width=1.\linewidth]{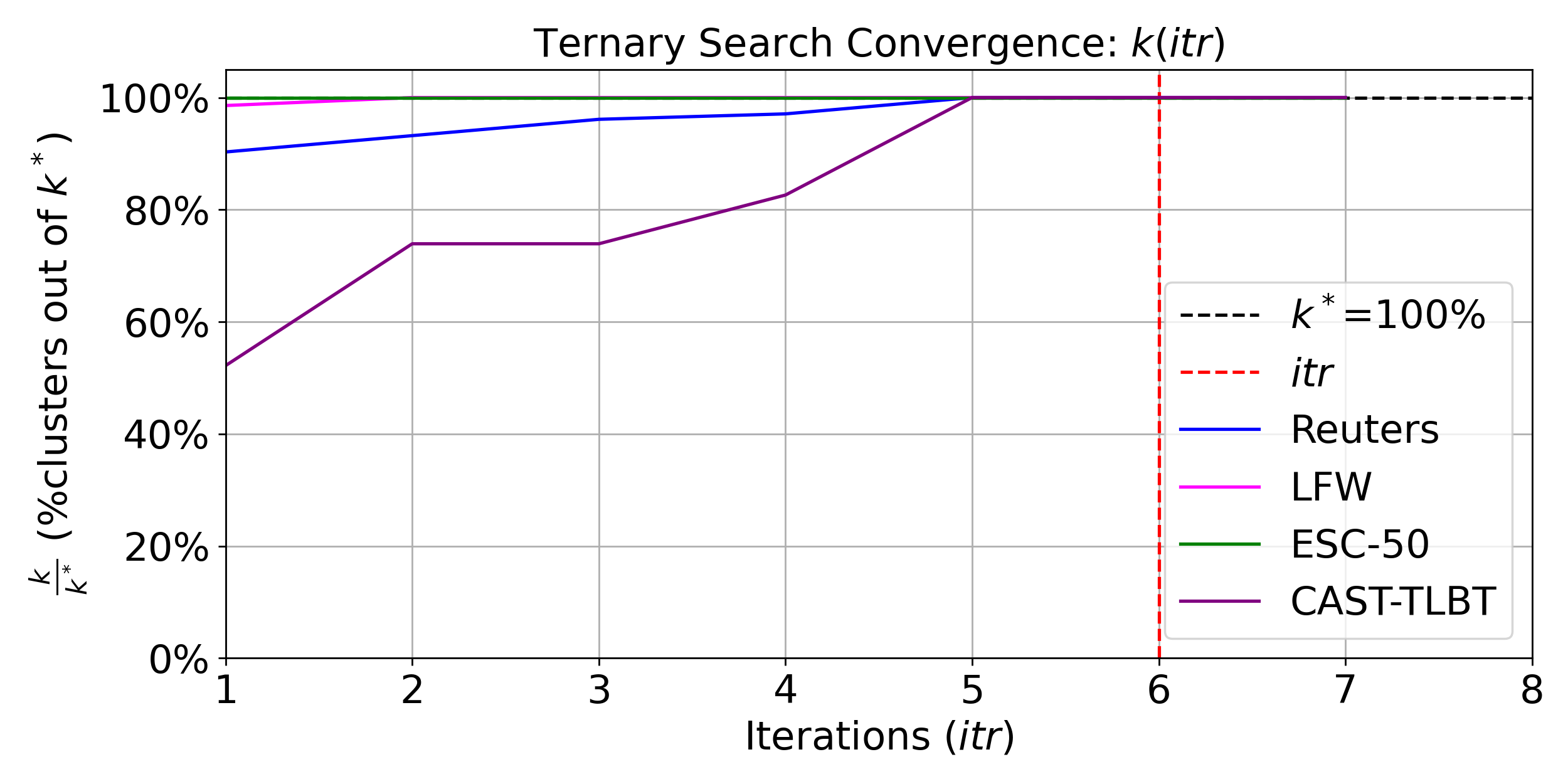}
        \subcaption{Ablation on the resulting approximation ratio ($k/k^*$) of TS as a function of $itr$.
        }
    \end{minipage}
\caption{Hyper-parameter tuning.}
\label{fig:hyperparametersAblation}
\end{figure*}

\textbf{Hyper-Parameter Tuning:} TSClustering relies on two hyper-parameters: $\alpha$ for sub-sampling $N$ and $D$ to give initial upper and lower bounds for $\varepsilon^*$, and $itr$ to bound the number of ternary search iterations within TS. Both parameters trade off precision and runtime. 
Fig.~\ref{fig:hyperparametersAblation}(a) illustrates the gap between $LB$ and $UB$ as a decreasing function of $\alpha$. We selected $\alpha=0.2$ to balance the gap size reduction with runtime.
Recall that we try to maximize $k$ and find its mode $k^*$. In Fig.~\ref{fig:hyperparametersAblation}(b), we illustrate the resulting approximation ratio ($\frac{k}{k^*}$) as a function of TS iterations over the Reuters, LFW, ESC, and CAST datasets.
We selected $itr=6$ since the approximation ratio converges just before this value.

\section{Conclusion}  
\label{sec:Conclusion}
This paper addresses the problem of parameter tuning in DBSCAN. We identify a unimodality property of \(k(\varepsilon)\), and support it empirically and theoretically. We find that maximizing $k(\varepsilon)$ provides a good clustering, and give a novel method to automatically find this $\varepsilon$ with an adapted version of the Ternary Search algorithm. Our empirical results on diverse datasets demonstrate improved precision and reduced noise, highlighting its potential for various data mining applications.

Future works may include: $\textbf{1.}$ creating a sub-linear estimator for $\varepsilon^*$, $\textbf{2.}$ improving the runtime with numerical optimization algorithms, which perhaps incorporate priors or gradients, $\textbf{3.}$ adapting the approach to multi-density distributions, and $\textbf{4.}$ scaling out for distributed big-data clustering.
\section{Acknowledgments}
This work was supported by the Joint NSFC-ISF Research Grant no. 3077/23.

\bibliographystyle{splncs04}

\section{Appendix}

\subsection{Additional Qualitative Evaluations of Near-Unimodality}
To support our qualitative claims of Near-Unimodality, in Fig.~\ref{fig:CASTSAILUnimodal} we illustrate $k(\varepsilon)$ over 15 additional datasets from SAIL AMCD~\cite{somandepalli2017unsupervised} and CAST~\cite{nir2022cast}, and plot the run of TSClustering over them.
SAIL-AMCD and CAST are collections of 15 animated videos of different styles where each video has its own set of detected characters embedded using the CAST embeddings of dimension $2,048$.

\begin{figure*}[tbp]
  \centering
  \includegraphics[width=\linewidth]{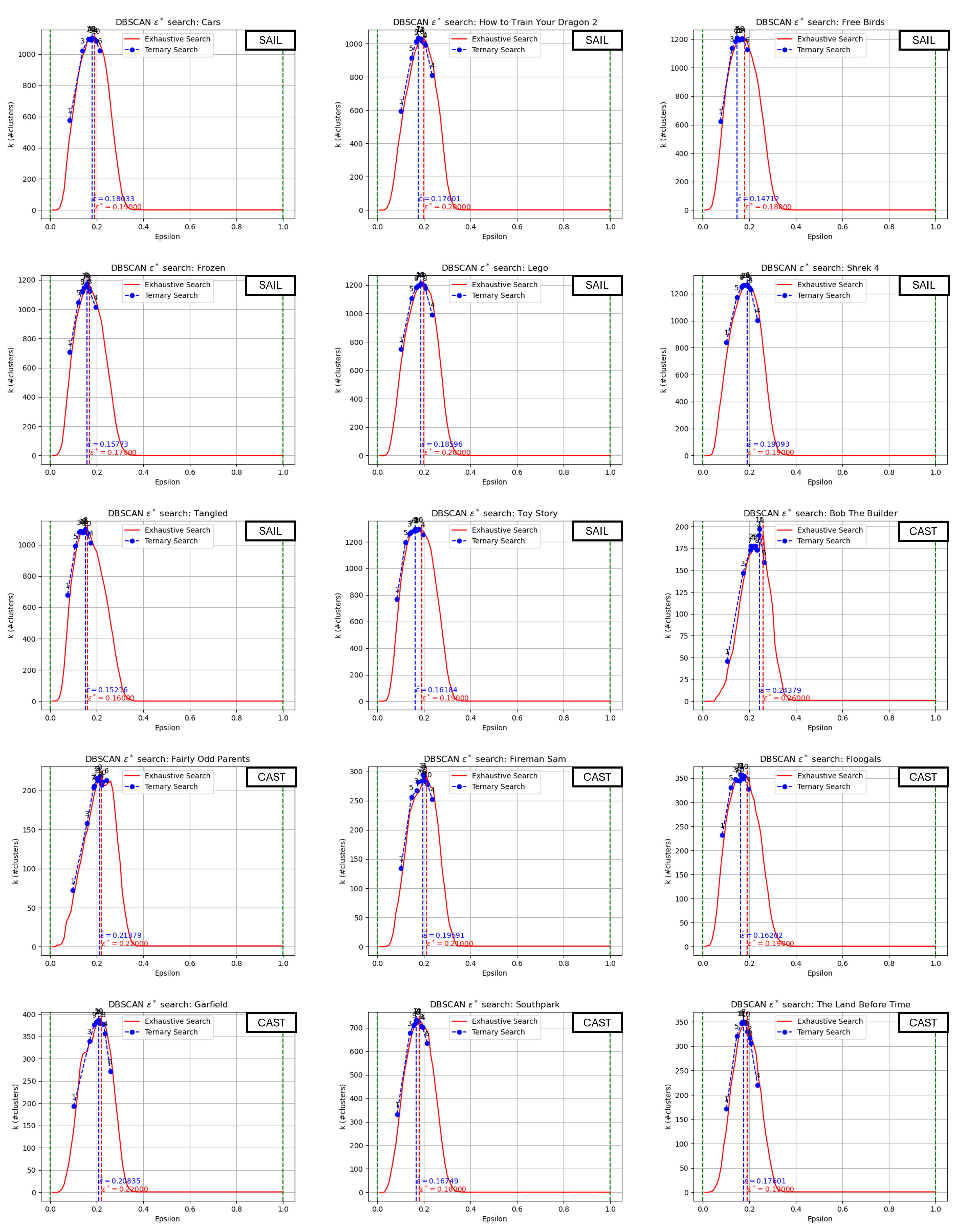}
  \caption{The function $k(\varepsilon)$ plotted over 15 datasets from SAIL-AMCD and CAST, each containing embeddings of multiple characters in a video. We observe a clear unimodal shape on all datasets.}
  \label{fig:CASTSAILUnimodal}
\end{figure*}

\end{document}